\documentclass{IEEEtran}

\usepackage{multicol}
\usepackage[bookmarks=true]{hyperref}
\usepackage{CJKutf8}
\usepackage{color}
\usepackage{enumerate}
\usepackage[ruled, linesnumbered]{algorithm2e}
\usepackage{amssymb, amsmath, amsthm}
\usepackage{stmaryrd}
\usepackage{graphicx}
\usepackage{setspace}
\usepackage{mathrsfs}
\usepackage{bm}
\usepackage[english]{babel}
\usepackage{blindtext}
\usepackage{epstopdf}
\usepackage{siunitx}
\usepackage{booktabs}
\usepackage{tikz}

\newtheorem{theorem}{Theorem}
\newtheorem{lemma}{Lemma}

\newtheorem{definition}{Definition}
\newtheorem{proposition}{Proposition}

\theoremstyle{definition}

\newcommand{\bfzero}{\mathbf{0}}

\newcommand{\bfc}{\bm{c}}

\newcommand{\bfy}{\bm{y}}

\newcommand{\bfX}{\mathbf{X}}

\newcommand{\bbR}{\mathbb{R}}

\newcommand{\calA}{{\cal A}}

\newcommand{\calC}{{\cal C}}
\newcommand{\calD}{{\cal D}}

\newcommand{\calS}{{\cal S}}

\newcommand{\calV}{{\cal V}}

\newcommand{\p}{\partial}

\def\ie{\emph{i.e.},~}
\def\eg{\emph{e.g.},~}

\def\d{\mathrm{d}}

\renewcommand\d[1]{{\rm d}{#1}}

\renewcommand{\vec}[1]{\mathbf{#1}}

\let\oldnl\nl
\newcommand{\nonl}{\renewcommand{\nl}{\let\nl\oldnl}}

\begin{document}
\title{Robotic manipulation of a rotating chain}

\author{Hung Pham \qquad Quang-Cuong Pham\\ \thanks{Hung Pham and
    Quang-Cuong Pham are with Air Traffic Management Research
    Institute (ATMRI) and Singapore Centre for 3D Printing (SC3DP),
    School of Mechanical and Aerospace Engineering, Nanyang
    Technological University, Singapore. This work was partially
    supported by grant ATMRI:2014-R6-PHAM (awarded by NTU and the
    Civil Aviation Authority of Singapore) and by the Medium-Sized
    Centre funding scheme (awarded by the National Research
    Foundation, Prime Minister's Office, Singapore).}}

\date{\today} 
\maketitle
\thispagestyle{empty}
\pagestyle{empty} 
\begin{abstract}

  This paper considers the problem of manipulating a uniformly
  rotating chain: the chain is rotated at a constant angular speed
  around a fixed axis using a robotic manipulator. Manipulation is
  quasi-static in the sense that transitions are slow enough for the
  chain to be always in ``rotational equilibrium''.  The curve traced
  by the chain in a rotating plane -- its shape function -- can be
  determined by a simple force analysis, yet it possesses a complex
  multi-solutions behavior typical of non-linear systems. We prove
  that the configuration space of the uniformly rotating chain is
  homeomorphic to a two-dimensional surface embedded in
  $\mathbb{R}^3$. Using that representation, we devise a manipulation
  strategy for transiting between different rotation modes in a stable
  and controlled manner. We demonstrate the strategy on a physical
  robotic arm manipulating a rotating chain. Finally, we discuss how
  the ideas developed here might find fruitful applications in the
  study of other flexible objects, such as circularly towed aerial
  systems, elastic rods or concentric tubes.

\end{abstract}
\section{Introduction}
\label{sec:intro}
An idle person with a chain in her hand will likely at some point
starts rotating it around a vertical axis, as in
Fig.~\ref{fig:rotate_by_hand}A. After a while, she might be able to
produce another mode of rotation, whereby the chain would curve
inwards, as in Fig.~\ref{fig:rotate_by_hand}B, instead of springing
completely outwards. With sufficient dexterity, she might even reach
more complex rotation modes, such as in
Fig.~\ref{fig:rotate_by_hand}C. Transitions into such complex rotation
modes are however difficult to reproduce reliably as instabilities can
quickly lead to unsustainable rotations
(Fig.~\ref{fig:rotate_by_hand}D). This paper investigates the
mechanics of the transitions between different rotation modes, and
proposes a strategy to perform those transitions in a stable and
controlled manner.

\subsection*{Motivations}

There are several reasons why this problem is hard to solve. First,
there are multiple solutions for a given control input (distance~$r$
between the attached end of the chain and the rotation axis, and
angular speed $\omega$). This ambiguity makes it difficult to devise a
manipulation strategy directly in the control space. Second, some
control inputs can quickly lead to ``uncontrollable'' behaviors of the
chain, as illustrated in in Fig.~\ref{fig:rotate_by_hand}D.

\begin{figure}[tp]
  \centering
  \includegraphics[width=0.5\textwidth]{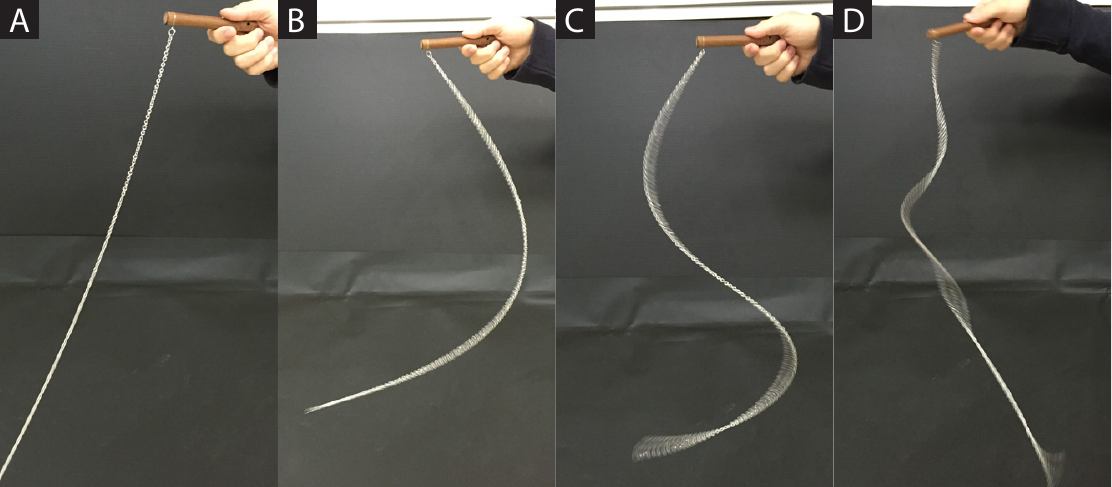}
  \caption{Manual rotation of a chain around a vertical
    axis. \textbf{A, B, C}: Uniform rotation modes 0, 1, 2
    respectively.  \textbf{D}: Unstable behavior.}
  \label{fig:rotate_by_hand}
\end{figure}  

The theoretical study of the rotating chain and, in particular, of its
rotation modes, has a long and rich history in the field of applied
mathematics~\cite{Kolodner1955, caughey1958whirling, Caughey1970,
  Wu1972, stuart1976steadily, russell1977equilibrium, Toland1979},
which we review in Section~\ref{sec:applied-mathematics}. Here, by
devising and implementing a \emph{manipulation strategy} to stably
transit between different rotation modes, we hope to provide a new,
robotics-enabled, understanding of this problem. Indeed, at the core
of our approach lie concepts specifically forged in the field of
robotics, such as ``configuration space'', ``stable configurations'',
``path-connectivity'', etc.

As opposed to rigid bodies, flexible objects are in general
characterized by an infinite number of degrees of freedom, which
entails significant challenges when it comes to manipulation. Specific
approaches have therefore been developed in the field of robotics to
study the manipulation of flexible objects, as reviewed in
Section~\ref{sec:robotic-manip}.

The above studies are motivated by a number of practical
applications. For the rotating chain in particular, applications
include aerial manipulation by Unmanned Air Vehicles (UAV), which has
recently received some attention, as discussed in more details in
Section~\ref{sec:aerial-manip}.



\subsection*{Contribution and organization of the paper}

Our contribution in this paper is threefold. First, we study the case
of arbitrary non-zero attachment radii~$r$
(Section~\ref{sec:forward_kinematics}). This extends and generalizes
existing works, which all focus on the case of zero attachment radius,
and sets the stage for stable transitions between different rotation
modes, which specifically require manipulating the attachment
radius. In particular, we determine in this section the number of
solutions to the shape equation for any given value $r$ and $\omega$.

Second, we show that the configuration space of the uniformly rotating
chain with variable attachment radius is homeomorphic to a
two-dimensional surface embedded in $\mathbb{R}^3$
(Section~\ref{sec:configuration_space}). We study the subspace of
stable configurations and establish that it is not possible to stably
transit between rotation modes without going back to the low-amplitude
regime.

Third, based on the above results, we propose a manipulation
strategy for transiting between rotation modes in a stable and
controlled manner (Section~\ref{sec:manipulation}). We show the
strategy in action in a physical experiment where a robotic arm
manipulates a rotating chain and makes it reliably transit between
different rotation modes.

Before presenting our contribution, we review related works
(Section~\ref{sec:related-works}) and recall Kolodner's equations of
motion of the rotating chain (Section~\ref{sec:background}). Finally,
we discuss possible applications and extensions and sketch some
perspectives for future work (Section~\ref{sec:conclusion}).

\section{Related works}
\label{sec:related-works}

The manipulation of the rotating chain is relevant to a number of
fields such as (i) applied mathematics, (ii) flexible object
manipulation in robotics, and (iii) aerial manipulation. We now review
the literature and describe the position of the current work with
respect to each of these fields.

\subsection{Theoretical studies of the rotating chain}
\label{sec:applied-mathematics}

In applied mathematics, the study of the rotating chain was initiated
in 1955 by a remarkable paper by
Kolodner~\cite{Kolodner1955}. Kolodner established the existence of
critical speeds $(\omega_i)_{i\in\mathbb{N}}$ such that there are no
uniform rotations if the angular speed $\omega<\omega_1$, and there
are exactly $n$ rotation modes for $\omega_n<\omega<\omega_{n+1}$.
In~\cite{caughey1958whirling}, Caughey studied the rotating chain with
small but non-zero attachment radii. The results obtained by Caughey
extend Kolodner's and agree with our study of the low-amplitude
regime. In~\cite{Caughey1970}, Caughey investigated the rotating chain
with both ends attached.  In \cite{stuart1976steadily}, Stuart
considered the original rotating chain problem using bifurcation
theory, and arrived at the same results as Kolodner. In~\cite{Wu1972},
Wu considered the large angular speeds regime. In~\cite{Toland1979},
Toland initiated a new approach based on the calculus of variation,
but did not obtain new significant results, as compared to Kolodner.

The common point of all previous works is that the chain is attached
to the rotation axis, or very close to
it~\cite{caughey1958whirling}. Yet, reliably observing and transiting
between different rotation modes precisely require using arbitrary
non-zero attachment radii $r$, the distance between the attached end
and the rotation axis. The current paper extends previous studies by
specifically considering arbitrary attachment radii.

\subsection{Robotic manipulation of flexible objects}
\label{sec:robotic-manip}

Within the field of robotics, the manipulation of flexible objects is
studied along two main directions. A first direction is topological:
one is mainly interested in the order and sequence of the manipulation
rather than in the precise behavior of the flexible object. Examples
include origami folding~\cite{balkcom2008robotic}, laundry
folding~\cite{miller2012geometric} or rope-knotting
\cite{wakamatsu2006knotting, Yamakawa2013a}.

The second research direction is concerned with the precise shape and
dynamics of the manipulated object. Within this research direction,
one can distinguish two main approaches. The first approach
discretizes the flexible object into a large number of small rigid
elements, and subsequently carries out finite-element calculations,
see \eg~\cite{Skop1971, russell1977equilibrium, murray1996trajectory}
for inextensible cables or~\cite{Gilbert2013a, Rucker2010a} for
concentric tube robots. This approach can be applied to any type of
flexible objects as long as a dynamical model is available. However,
it usually yields no \emph{qualitative} understanding of the
manipulation. For example, while finite-element calculations can
compute the shape of the rotating chain for various control inputs,
they can establish neither the existence of different rotation modes,
nor the manipulation strategies to transit between different modes.

By contrast, the second approach considers the flexible object as the
solution of a (partial) differential equation and tries to establish
qualitative properties of this solution.  While this approach is
harder to put in place -- usually because of the complex mathematical
calculations and concepts involved -- it can lead to stunning and
insightful results. For example, Bretl and colleagues established that
the configuration space of the Kirchhoff elastic rod is of
dimension~6~\cite{bretl2014quasi} and that it is
path-connected~\cite{borum2015free}. Such results would be impossible
to obtain via finite-element methods.

The present study of the rotating chain is inscribed within this
analytical approach. From the dynamic model of the rotating chain, we
investigate qualitative properties of its configuration space:
dimension, connectivity, and stability. These properties are in turn
crucial to devise a manipulation strategy to stably transit between
different rotation modes.

\subsection{Aerial manipulation}
\label{sec:aerial-manip}

Although the study of the rotating chain first stemmed out of
scientific curiosity, it has recently found applications in aerial
manipulation. In~\cite{murray1996trajectory, Williams2007}, the
authors considered a fixed-wing aircraft towing a long cable whose other
end is free. The circular flying pattern imprints a pseudo-stationary
shape to the cable, which in turn allows precisely controlling the
position of the free end. Practical applications of this scheme
include remote sensing in isolated areas~\cite{Williams2007}, payload
delivery and pickup~\cite{Skop1971, merz2016feasibility,
  Williams2007}, or more recently, recovery of micro air
vehicles~\cite{Colton2011, nichols2014aerial, Sun2014}. In the latter
application, the micro vehicles are able to attach themselves to the
towed end, which moves at a relatively slower speed than that of the
aircraft. The recent surge of interest in Unmanned Air Vehicles (UAVs)
also offers many potential applications: \cite{merz2016feasibility}
studies a single UAV flying circularly while towing a cable,
\cite{sreenath2013trajectory} deals with general (non-circular) aerial
manipulation, while~\cite{michael2011cooperative} targets cooperative
manipulation using a team of UAVs.



The above works are based on dynamic simulation~\cite{Skop1971,
  russell1977equilibrium, Williams2008}, or numerical optimal
control~\cite{Sun2014, Williams2007a}. Physical experiments were found
to agree with simulations~\cite{Williams2007}. However, there are a
number of questions these works are unable to address, for instance:
(i) under which conditions are there multiple solutions to the same
set of controls (fly radius and angular speed)?  (ii) how to avoid or
initiate ``jumps'' between different quasi-static rotational
solutions~\cite{russell1977equilibrium}? Here, we precisely answer
these questions for the case of a simple rotating chain, without
considering aerodynamic drag or end mass. We also discuss how the
method can be extended to include these effects, offering thereby
solid theoretical foundations for developing safe and stable
applications in circular aerial manipulation.


\section{Background and problem setting}
\label{sec:background}

\subsection{Equations of motion of the rotating chain}

Here we recall the main equations governing the motion of the rotating
chain initially obtained by Kolodner~\cite{Kolodner1955}.
Fig.~\ref{fig:schema} depicts an inextensible and homogeneous chain of
length $L$ and linear density $\mu$ that rotates around a vertical
$Z$-axis. One end of the chain is maintained at the attachment radius
$r$ from the rotation axis, while the other end is free. Note that the
case of a chain with tip mass can be reduced to this case, see
Appendix~\ref{sec:tip-mass}.


\begin{figure}[htp]
  \centering
  \includegraphics[width=0.3\textwidth]{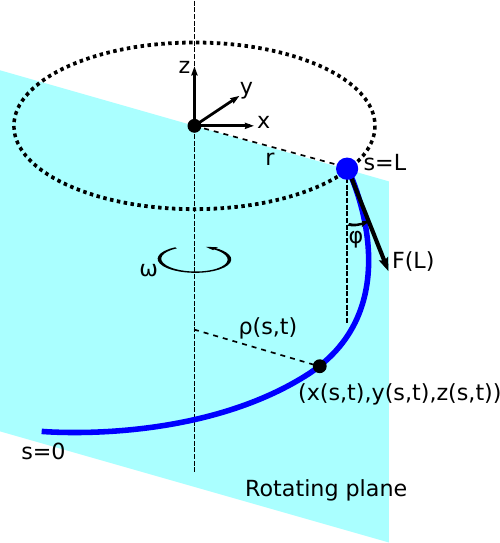}
  \caption{A chain rotating around a fixed vertical axis. At a time
    instant $t$, the chain describes a 3D curve parameterized by $s$:
    $s=0$ at the free end, $s=L$ at the attached end, where $L$ is the
    length of the chain.}
\label{fig:schema}
\end{figure}

Let $\vec x(s,t) := {[x(s,t),y(s,t),z(s,t)]}^\top\in \bbR^3$ denote a
length-time parameterization of the chain where $s$ equals zero at the
free end and equals $L$ at the attached end
(Fig.~\ref{fig:schema}). Next, let $F(s,t)\geq 0$ be the tension of
the chain. Neglecting aerodynamic effect, one writes the equation of
motion for the chain as
\begin{equation}
  \label{eq:eom}
  \mu \ddot{\vec x} = (F{\vec x}')' + \mu \vec{g},
\end{equation}
where $\dot \Box$ and $\Box'$ denote differentiation with respect to
$t$ and $s$ respectively; $\vec g:=[0, 0, -g]^{\top}$ is the gravitational
acceleration vector. The inextensibility constraint can be written as
\begin{equation}
  \label{eq:inex}
  \|\vec x(s, t)'\|_{2} = 1.
\end{equation}

We seek solutions that are \emph{uniform rotations}; those which have
constant shape in a plane that rotates around the $Z$-axis. In this
case, the motion of the chain becomes
\begin{equation}
  \begin{aligned}
    x(s,t) & = \rho(s) \cos(\omega t),\\
    y(s,t) & = \rho(s) \sin(\omega t),\\
    z(s,t) & = z(s),
  \end{aligned}
\end{equation}
where the function $\rho(s)$ is called the \emph{shape function} of
the chain. Directly from inextensibility constraint~\eqref{eq:inex},
we have:
\begin{equation}
  \label{eq:11}
  \|(\rho(s), z(s))\|_2 = \sqrt{\rho'(s)^2 + z'(s)^2} = 1.
\end{equation}
Also, the tension of the chain $F(s, t)$ is time independent.

Substituting the above expressions into Eq.~\eqref{eq:eom} yields
\begin{align}
  \label{eq:x1z}
  (F\rho')' + \mu\rho \omega^2 &= 0, \\
  \label{eq:2}
  (Fz')'  - \mu g &= 0,
\end{align}
where $F, \rho, z$ are functions of $s$.  Integrating Eq.~\eqref{eq:2}
and noting that the tension at the free end vanishes (\ie
$F(0)=0$) yield
\begin{equation}
  \label{eq:24}
  Fz' = \int_0^s\mu g \; \d \lambda= \mu g s.
\end{equation}
Next, by the inextensibility constraint~\eqref{eq:11}, we have
\begin{equation}
  \label{eq:F}
  F =\frac{\mu g s}{z'} =  \frac{\mu g s}{\sqrt{1-\rho'^2}}.
\end{equation}
Substituting Eq.~(\ref{eq:F}) into Eq.~(\ref{eq:x1z}) yields the
governing equation for the shape function $\rho(s)$
\begin{equation}
  \label{eq:x1}
  \frac{\d {}}{\d s} \left( \frac{\mu g s}{\sqrt{1-\rho'^2}}\rho' \right)
  + \mu\rho\omega^2
  = 0
\end{equation}
subject to the following boundary condition 
\begin{equation}
  \label{eq:1}
  \rho(L) = r.
\end{equation}

Remark that we have applied two boundary conditions: (i) tension at
the free end must be zero: $F(0, t)=0$ for any $t$; and (ii)
$\vec x(L, t)$ equals the reference trajectory traced by the
robotic manipulator (or the aircraft's trajectory in the towing
problem).

\subsection{Problem formulation}

We can now define the \emph{configurations} and the \emph{control
  inputs} of a rotating chain.

\begin{definition} \label{def:conf} (Configuration) A configuration of
  the rotating chain is a pair $q:=(\omega,\rho)$, where
  $\omega \geq 0$ is a rotation speed and $\rho$ is a shape function
  satisfying the governing equation~(\ref{eq:x1}) and that
  $\rho(0)\geq 0$. The set of all such configurations is called the
  configuration space of the rotating chain and denoted $\calC$.
\end{definition}

\begin{definition} (Control input) A control input is a pair
  $(r, \omega)$, where $r\geq 0$ is an attachment radius and
  $\omega\geq 0$ is a rotation speed. The set of all inputs is called
  the control space and denoted $\calV$. If equation~(\ref{eq:x1}) has
  non-trivial solutions with boundary conditions and parameters
  defined by the input ($r, \omega$) then the input is called
  \emph{admissible}.
\end{definition}


Note that the condition $\rho(0) \geq 0$ in Definition \ref{def:conf}
identifies duplicate solutions. Any configuration $(\omega,\rho)$
corresponds to two possible solutions: one has shape function $\rho$
and one has shape function $-\rho$, both rotate at angular speed
$\omega$. The later solution can be obtained by rotating the former
solution by $180$ degrees. A similar remark applied to the definition
of the control space $\calV$ where we require positive attachment
radius.


We can formulate the chain manipulation problem as follows: given a
pair of starting and goal configurations
$(q_\mathrm{init}, q_\mathrm{goal})$ find a control trajectory
$(0, 1) \rightarrow \calV$ that brings the chain from
$q_\mathrm{init}$ to $q_\mathrm{goal}$ without going through
instabilities (instabilities will be discussed in
Section~\ref{sec:stability_analysis}).


\section{Forward kinematics of the rotating chain with non-zero
  attachment radius}
\label{sec:forward_kinematics}

\subsection{Dimensionless shape equation} 
\label{sec:dimless}

Still following Kolodner, we convert Eq.~\eqref{eq:x1} into a
dimensionless equation, more appropriate for subsequent
analyses. Consider the changes of variable
\begin{equation}
  \label{eq:change_of_var}
  u := \frac{\rho'}{\sqrt{1-{\rho'}^2}}\frac{s\omega^2}{g}, \quad \bar
  s := \frac{s\omega^2}{g},
\end{equation}
which by combining with Eq.~\eqref{eq:x1} leads to
\begin{equation}
\label{eq:rho2}
\frac{\d {u}}{\d {\bar s}} + \rho \frac{\omega^2}{g} = 0.
\end{equation}


One can now differentiate Eq.~\eqref{eq:rho2} with respect to
$\bar s$ to arrive at
\[
    \frac{\d{}^2  }{\d{\bar s}^2 } u + \rho' = 0,
\]
which is combined with the relation
\begin{equation}
\label{eq:rho1}
  \rho' =\frac{u }{\sqrt{{\bar s}^2+u^2}}
\end{equation}
to yield the dimensionless differential equation
\begin{equation}
  \label{eq:new}
  \frac{\d{}^2  }{\d{\bar s}^2 } u(\bar s) +
  \frac{u(\bar s)} {\sqrt{\bar s^2+{u(\bar s)}^2}} = 0. 
\end{equation}

We first consider the boundary condition at $\bar s = 0$. By
definition of $u$, one has $u(0)=0$.  The end boundary condition
$\rho(L)=r$ implies that
\begin{equation}
\label{eq:boundary_2}
u'\left(L\frac{\omega^2}{g}\right)= -r\frac{\omega^2}{g},
\end{equation}
where $\Box'$ denotes in this context differentiation with respect to
$\bar s$. 

We summarize the boundary conditions on $u$ as
\begin{equation}
  \label{eq:boundary}
  u(0) = 0, \quad u'(\bar L) = \bar r,
\end{equation}
where
\begin{equation}
  \label{eq:defa}
   \bar L:=L\omega^2/g, \quad
  \bar r:=-r\omega^2/g. 
\end{equation}
This is the standard form of a Boundary Value Problem (BVP).

\noindent \textbf{Remark} Denote by $\rho_{0}$ the distance from the
free end to the $Z$-axis. Using Eq.~\eqref{eq:rho2}, we have
\begin{equation}
\label{eq:boundary_1}
u'(0)= a, 
\end{equation}
where $a = -\rho_0 \omega^2/ g$.

\noindent \textbf{Remark} Applying L'Hôpital rule twice, one finds
that
\[
\lim_{\bar s\to 0}\frac{u(\bar s)} {\sqrt{\bar s^2+{u(\bar s)}^2}} = \frac{a}{\sqrt{1+a^2}}.
\]
Thus, the differential equation~(\ref{eq:new}) is well-defined at
$\bar s=0$. 
\subsection{Shooting method}
\label{sec:shooting}

\begin{figure}[t]
  \centering
  \begin{tikzpicture}
    \node[anchor=south west,inner sep=0] (image) at (0,0) {
      \includegraphics[width=0.3\textwidth]{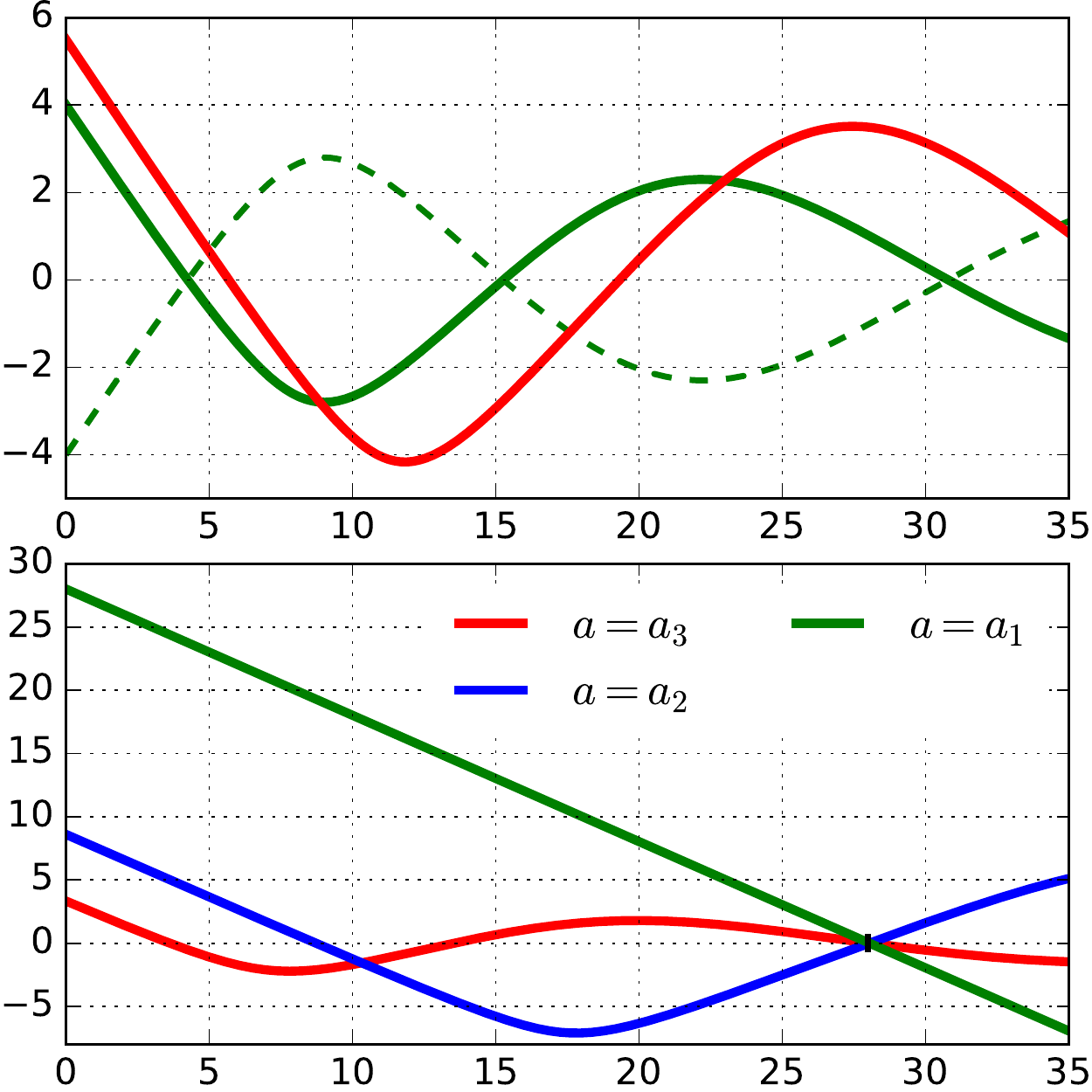}
    };  
    \begin{scope}[x={(image.south east)},y={(image.north west)}]
      \node at (0.53, -0.04) {$\bar s$};
      \node at (-0.04, 0.29) {$u'$};
      \node at (-0.04, 0.79) {$u'$};

      \filldraw [color=black!60, fill=red!5] (0.795, 0.135) circle
      (0.07cm) node [anchor=south, color=black] {$(\bar L, 0)$};
      \filldraw [color=black!60, fill=red!5] (0.67, 0.835) circle
      (0.07cm) node [anchor=south east, color=black] {$(\bar L, \bar r)$};
      
      \node at (1.02, 0.96) {\textbf{A}};
      \node at (1.02, 0.46) {\textbf{B}};
    \end{scope}

  \end{tikzpicture}
  \caption{\textbf{A}: Shooting from different initial guesses of
    $u'(0)=a$. There might be more than one initial value
    (green and red) that satisfy the end condition
    $u'(\bar L) = \bar r$.  \textbf{B}: $a_1, a_2, ...$ are
    different initial values of $u'(0)$ that yield
    $u'(\bar L)=0$; $a_i$ denotes the initial guess such that the
    $i$-th zero of $u'$ coincides with $\bar L$}
  \label{fig:2Dshooting}

\end{figure}
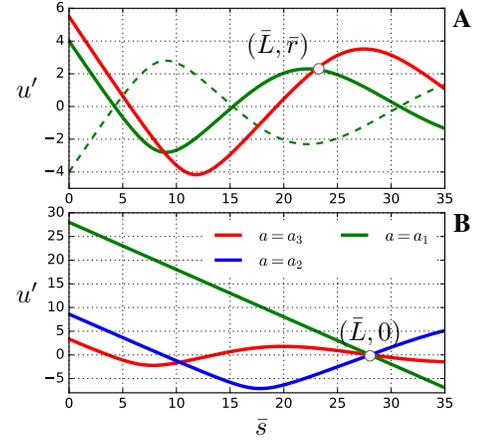

We numerically solve the BVP posed in the last section using the
\emph{simple shooting method}~\cite{Stoer1982}.  Given a control input
$(r, \omega)$, the method finds resulting configurations as follows:
\begin{enumerate}
\item[1.] compute $(\bar r, \bar L)$ from $(r, \omega)$ using
  Eq.~\eqref{eq:defa};
\item[2.] repeat until convergence:
  \begin{enumerate}
  \item[2a.] guess an initial value $a\in \bbR$ for $u'(0)$ or
    use the value from the last iteration;
  \item[2b.]  integrate Eq.~\eqref{eq:new} from the initial
    condition $(u(0),u'(0))=(0,a)$ at $\bar s = 0$ to
    $\bar s=\bar L$;
  \item[2c.] check whether $u'(\bar L)=\bar r$;
  \item[2d.] if not, refine the guess $a$ by \eg Newton's method;
  \end{enumerate}
\item[3.] recover $\rho(s)$ from ${u'}_{\rm{last\_iter}}(\bar s)$.
\end{enumerate}
One can then recover $z(s)$ using $\rho(s)$, the inextensibility
constraint~\eqref{eq:11}, the boundary condition $z(L)=0$ and the fact
that $z'(s)\geq 0$ (See Eq.~\eqref{eq:24}).  Also, for a given tuple
$(\bar r, \bar L)$, there might be multiple solutions to the BVP which
translates to multiple configurations for a given control input
(Fig.~\ref{fig:2Dshooting}\textbf{A}).

\noindent \textbf{Remark} It is straightforward to see that if
$u(\bar s)_{\bar s\in [0, \bar L]}$ is a solution of
Eq.~\eqref{eq:new}, then $- u(\bar s)_{\bar s\in [0, \bar L]}$ is also
a solution. Therefore there is no loss of generality to consider only
non-negative values of $a$, as integrating from $-a$ leads to the same
configuration.

\subsection{Number of configuration}
\label{sec:number-configuration}


\begin{figure}[tp]
  \centering
  \begin{tikzpicture}
    \node[anchor=south west,inner sep=0] (image) at (0,0) {
      \includegraphics[width=0.3\textwidth]{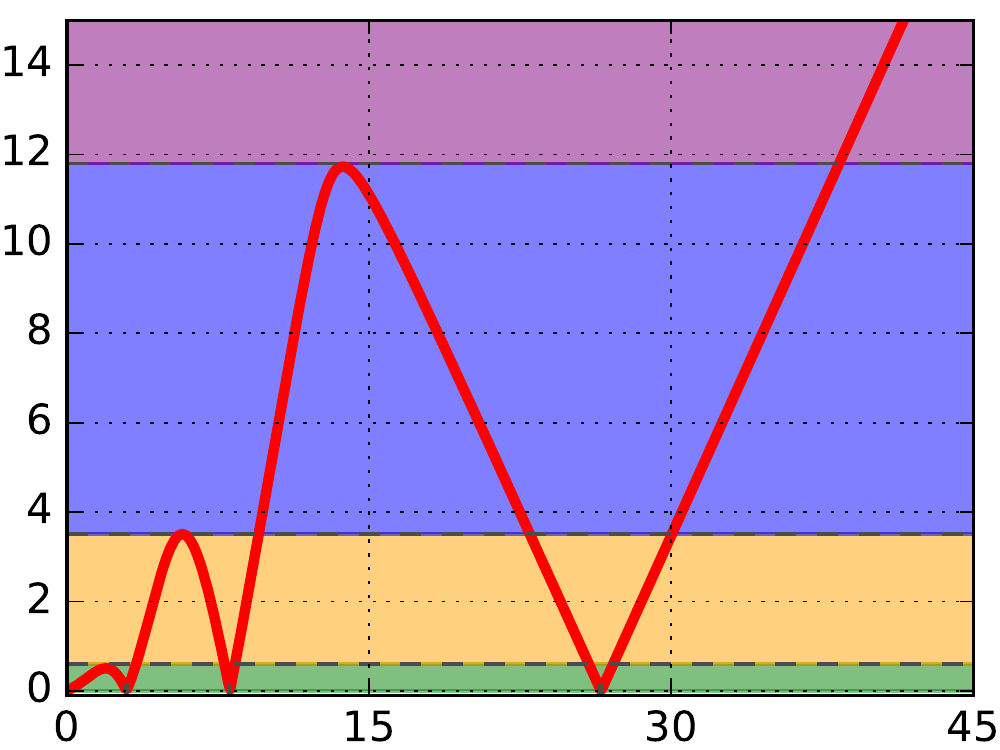}
    };  
    \begin{scope}[x={(image.south east)},y={(image.north west)}]
      \draw[very thick, color=green] (0.155, 0.25) -- (0.66, 0.25);
      \filldraw [fill=green] (0.162,0.25) circle [radius=0.055cm];
      \filldraw [fill=green] (0.205,0.25) circle [radius=0.055cm];
      \filldraw [fill=green] (0.255,0.25) circle [radius=0.055cm];
      \filldraw [fill=green] (0.545,0.25) circle [radius=0.055cm];
      \filldraw [fill=green] (0.66,0.25) circle [radius=0.055cm];

      \node  at (-0.1, 0.5) {$|u'_a(\bar L)|$};
      \node at (0.5, -0.05) {$|a|$};
      \node [align=left] (a) at (0.83, -0.035) {$7$ solutions};
      \draw [->, thick] (a) -- (0.83, 0.1);
      \node [align=left] at (0.83, 0.225) {$5$ solutions};
      \node [align=left] at (0.83, 0.715) {$3$ solutions};
      \node [align=left] at (0.83, 0.925) {$1$ solution};

      \node at (0.11, 0.16) {$\bar r_3$};
      \node at (0.19, 0.35) {$\bar r_2$};
      \node at (0.35, 0.82) {$\bar r_1$};

      \node at (0.14, 0.025) {$a_3$};
      \node at (0.24, 0.025) {$a_2$};
      \node at (0.6, 0.025) {$a_1$};
      
    \end{scope}

  \end{tikzpicture}
  \caption{The graph of $|u'_a(\bar L)|$ versus $|a|$.  The
    main text shows that if $\bar r_{i+1}<|\bar r|<\bar r_i$, then there
    are $2i+1$ non-trivial solutions. The green line illustrates the
    case $\bar r_3<|\bar r|<\bar r_2$ where there are 5 non-trivial
    solutions (green disks).}
  \label{fig:kappa_i}
\end{figure}
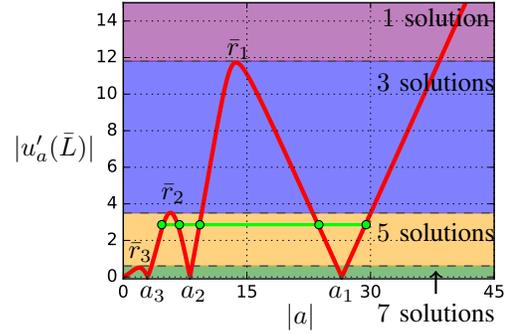

We now analyze the number of solutions for different parameters.
Denote by $u'_a(\bar s)$ the function $u'(\bar s)$
obtained by integrating from $(u(0),u'(0))=(0,a)$.
Following Kolodner, let $z_i(a)$ be the $i$-th zero
of $u'_a(\bar s)$.  The function $z_i(a)$ has the following
properties (Theorem 2~\cite{Kolodner1955}):
\begin{itemize}
\item $z_i(a)$ is well-defined for all $i \in \mathbb{N}$ and is a
  strictly increasing function of $a$ over $(0, +\infty)$;
\item  $\lim_{a\rightarrow 0}z_i(a)=h_i^2 / 4=:\lambda_i$ where $h_i$ is the
  $i$-th zero of the Bessel function $J_0$ (Appendix~\ref{sec:lowamp});
\item $\lim_{a\rightarrow +\infty }z_i(a) = +\infty$.
\end{itemize}

Next, let us define $a_i$ as the absolute value of $a$ such that
$z_i(a)=\bar L$, \ie
\begin{equation}
  \label{eq:12}
  a_i := |z_i^{-1}(\bar L)|.
\end{equation}
By the properties of $z_i$, $a_i$ exists if and only if
$\lambda_i \leq \bar L$ and when it exists, it is unique since
$z_i(a)$ is a strictly increasing function of~$a$.


Fig.~\ref{fig:2Dshooting}\textbf{B} shows the construction of $a_1$,
$a_2$, $a_3$. One can also observe that the $a_i$'s form a decreasing
sequence, \ie
\[
  a_1 > a_2 > a_3 > \dots > a_n,
\]
where $n$ is the largest $i$ so that
$ \lambda_i \leq \bar L$.

We now turn to the general case where $\bar r$ is not necessarily
zero.  Consider fixed parameters $(\bar r, \bar L)$, it can be seen
that the number of configurations equals the number of intersections
that the $u'_a(\bar L)$ versus $a$ graph makes with the
horizontal lines $u'_a(\bar L) = \bar r$ and
$u'_a(\bar L) = -\bar r$.

In fact, we can simplifies further.  Since $\bar r$ and $- \bar r$
refer to the same radius and that $\rho(s)$ and $-\rho(s)$ refer to
the same shape function, the number of intersections the
$|\bar\rho_a(\bar L)|$ versus $|a|$ graph makes with the horizontal
line $|u'_a(\bar L)|= |\bar r|$ equals the number of
configurations (Fig.~\ref{fig:kappa_i}).

By inspecting Fig.~\ref{fig:kappa_i}, $|u'_a(\bar L)|$ is zero at
$a_i$ and $a_{i+1}$; note moreover that $|u'_a(\bar L)|$ increases as
$a$ increases from $a_{i+1}$, reaches a maximum at some $a_i^*$, and
then decreases as $a$ increases from $a_i^*$ to $a_i$~\footnote{This
  claim is based on numerical observations.}.  Let us denote the
maximum reached by $|u'_a(\bar L)|$ between $a_i$ and $a_{i+1}$ by
$\bar r_i$, \ie
\begin{align}
  \label{eq:13}
\bar r_i&:=|u'_{a_i^*}(\bar L)| = \max_{a_{i+1}<a<a_i} |u'_a(\bar L)|, \quad\text{for } i < n;\\
\bar r_n&:=\max_{0<a<a_n} |u'_a(\bar L)|.
\end{align}
One can next observe that the $\bar r_i$'s form a decreasing
sequence\,\footnote{We have not yet been able to prove
  rigorously that the sequence is indeed decreasing.\label{fn:1}}, \ie
\[
\bar r_1 > \bar r_2 > \bar r_3  > \dots > \bar r_n.
\] 

One can now state the following proposition, whose proof results
directly from the examination of Fig.~\ref{fig:kappa_i}.

\begin{proposition}
  \label{prop:nsol}
  Let $n$ be the largest $i$ so that $\lambda_i \leq \bar L$.  The
  number of non-trivial configurations of an uniformly rotating chain
  depends on $|\bar r|$ as follows:
  \begin{enumerate}
  \item if $|\bar r|=0$, there are $n$ non-trivial solutions;
  \item if $0 < |\bar r|<\bar r_n$,
    there are $2n+1$ non-trivial solutions;
  \item if $\bar r_{i+1}<|\bar r|<\bar r_i$ for $i\in[1, n-1]$,
    there are $2i+1$ non-trivial solutions;
  \item if $|\bar r|=\bar r_i$ for $i=[1,n]$, there are $2i$ non
    trivial solutions;
  \item if $|\bar r|>\bar r_1$, there is one non-trivial solution.
  \end{enumerate}
\end{proposition}


\subsection{Rotation modes}
\label{sec:modes}

By the change of variable~\eqref{eq:change_of_var},
$u'=\rho\omega^2/g$, the number of zeros of $u'(\bar s)_{\bar
  s\in(0,\bar L)}$ corresponds to the number of times the chain
crosses the rotation axis. We can now give an operational definition
of rotation modes.

\begin{definition}[Rotation modes]
  A chain is said to be rotating in mode $i$ if its shape crosses the
  axis exactly $i$ times or, in other words, if the function
  $u'(\bar s)_{\bar s\in(0,\bar L)}$ has exactly $i$ zeros.
\end{definition}

Let us re-interpret Prop.~\ref{prop:nsol} in terms of rotation
modes. Consider a positive $\bar r$ verifying
$\bar r_{i+1}<\bar r<\bar r_i$. In Fig.~\ref{fig:kappa_i}, the
horizontal line $|u'(\bar L)| = \bar r$ intersects the graph of
$|\bar \rho_a(\bar L)|$ versus $|a|$ at $2i+1$ points. Call the
$X$-coordinates of these points $b_1>b_2>\dots>b_{2i+1}$. Remark that
\begin{itemize}
\item $b_1>a_1$, thus by definition of $a_1$, the function
  $u'_{b_1}(\bar s)$ has no zero in $(0,\bar L)$, \ie the chain
  rotates in mode 0;
\item $a_1>b_2>b_3>a_2$, thus by definition of $a_1,a_2$, the
  functions $u'_{b_2}(\bar s)$ and $u'_{b_3}(\bar s)$ have
  each one zero in $(0,\bar L)$, \ie the chain rotates in mode 1;
\item more generally, for any $k\in [1,i]$,
  $a_{k-1}>b_{2k}>b_{2k+1}>a_k$, thus by definition of $a_{k-1},a_k$,
  the functions $u'_{b_{2k}}(\bar s)$ and
  $u'_{b_{2k+1}}(\bar s)$ have each $k$ zeros in $(0,\bar L)$,
  \ie the chain rotates in mode $k$.
\end{itemize}

Figure~\ref{fig:multi_solutions} illustrates the above discussion for
$i=2$. 

\begin{figure*}[htp]
  \centering
  \includegraphics[width=\textwidth]{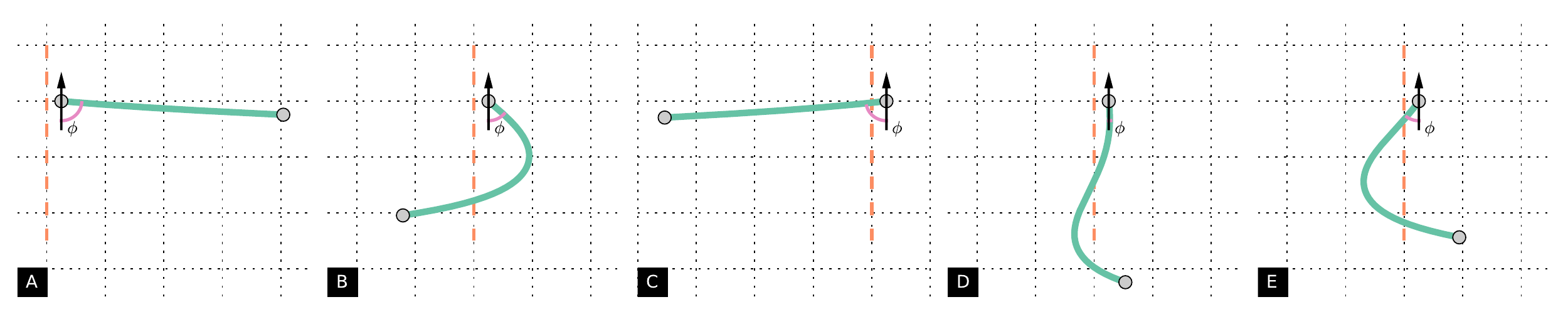}
  \caption{Rotation modes for $\bar r$ where $\bar r_3< \bar r<\bar
    r_2$ ($i=2$). According to Proposition~\ref{prop:nsol}, there are
    $2i+1=5$ solutions, depicted in A--E. A rotation is said to be in
    mode $i$ if the chain shape crosses the rotation axis (dashed
    line) $i$ times. \textbf{A}: solution in mode~0, corresponding to
    $b_1$ (for the explanation of the numbers $b_i$, see main
    text). \textbf{B, C}: solutions in mode~1, corresponding to
    $b_2,b_3$. \textbf{D, E}: solutions in mode~2, corresponding to
    $b_4,b_5$. In addition, the analysis of
    Section~\ref{sec:stability_analysis} shows that A, B and D are
    stable, while C and E are unstable.}
  \label{fig:multi_solutions}
\end{figure*}

\section{Analysis of the configuration space of the rotating chain}

\label{sec:configuration_space}

In the previous section, we have established a relationship between
the control inputs and the configurations. Here, we investigate the
properties of the configuration space and of the subspaces of stable
configurations. In particular, a crucial question for manipulation,
which we address, is whether the stable subspace is \emph{connected},
allowing for stable and controlled transitions between different
modes.

\subsection{Parameterization of the configuration space}

From now on, we make two technical assumptions: (i) the distance
$\rho(0)$ from the free end of the chain to the rotation axis is
upper-bounded by some $\rho_{\max}$; (ii) the rotation speed $\omega$
is upper-bounded by some $\omega_{\max}$. Note that these two
assumptions do not reduce the generality of our formulation since they
simply assert that there exist some finite bounds, which could be
arbitrarily large. From Eq.~(\ref{eq:defa}) and~\eqref{eq:boundary_1},
the two assumptions next imply that $a$ and $\bar L$ are upper-bounded
by some constants $a_{\max}$ and $\bar L_{\max}$. We can now prove a
first characterization of the configuration space.

\begin{proposition}[and definition]
  \label{prop:A_to_C}
  Define the parameter space $\calA$ by
    $$
    \calA := (0, a_{\max}) \times (0, \bar L_{\max}).
    $$
  There exists a homeomorphism $f:\calA\to\calC$.
\end{proposition}

This proposition implies that, despite (a) the potentially infinite
dimension of the space of all shape functions $\rho$ and (b) the
one-to-many mapping between control inputs and configurations, the
configuration space of the rotating chain is actually of dimension 2
and has a very simple structure.  Note that $\calA$ is essentially a
2D box.

The first dimension, $a$, is proportional to the distance of the free
end to the rotation axis. Thus, choosing the free end rather than the
attached end as reference point allows finding a one-to-one mapping
with the shape function. The second dimension, $\bar L$, is defined by
$\bar L:=L\omega^2/g$. Since the length $L$ of the chain is fixed,
$\bar L$ changes as a function of the angular speed $\omega$.

To simplify the notations, we define $\vec u:=(u,u')$ and rewrite
Eq.~\eqref{eq:new} as a dimensionless ODE
\begin{equation}
  \label{eq:X}
  \frac{\d {\vec u}}{\d {\bar s}} = \bfX(\vec u,\bar s).
\end{equation}
We can now give a proof for Proposition \ref{prop:A_to_C}. 

\begin{proof}
  The mapping $f$ is essentially the shooting method described in
  Section~\ref{sec:shooting}. Given a pair $(a,\bar L)\in \calA$, we
  first obtain $\omega$ from $\bar L$ using the relationship
  $\bar L = L\omega ^2/g$. Next, we integrate the ODE~\eqref{eq:X}
  from the initial condition
  \[
    \vec u(0) = ( 0, a )
  \]
  until $\bar s=\bar L$ to obtain $u'(\bar s)$ for
  $ \bar s \in (0, \bar L)$.  Finally, we obtain $\rho$ from
  $u'$ using Eq.~\eqref{eq:change_of_var}. 

  (1) Surjectivity of $f$. Let $(\omega,\rho)\in \calC$. Since $\rho$
  verifies~(\ref{eq:x1}), one can perform the change of
  variables~(\ref{eq:change_of_var}) and obtain $u$ and
  $u'$. Next, consider $a=u'(0)$ and
  $\bar L=L\omega ^2/g$. One has clearly $a\in(0,a_{\max})$,
  $\bar L\in(0,\bar L_{\max})$, and $f((a,\bar L))=(\omega,\rho)$.

  (2) Injectivity of $f$. Assume that there are $(a_1, \bar L_1) \neq
  (a_2, \bar L_2)$ such that $f(a_1,\bar L_1) = f(a_2,\bar L_2) =
  (\omega,\rho)$. One has $a_1=a_2=-\rho(0)\omega^2/g$ and $\bar
  L_1=\bar L_2=L\omega^2/g$, which implies the injectivity.

  (3) Continuity of $f$. We show in the Appendix~\ref{sec:ODE} that
  the ODE~\eqref{eq:X} is Lipschitz. It follows that the function
  $u'(\bar s)$ for $0 \leq \bar s \leq \bar L$ depends continuously on
  its initial condition, which implies that $\rho(s)$ depends
  continuously on $a$.

  (4) Continuity of $f^{-1}$. It can be seen from the injectivity
  proof that $a$ and $\bar L$ depend continuously on $\omega$ and
  $\rho(0)$, and the latter depends in turn continuously on $\rho$.
\end{proof}

Next, we establish a homeomorphism between the parameter space and a
smooth surface in 3D, which allows an intuitive visualization of the
configuration space.

\begin{proposition}[and definition]
  \label{prop:A_to_S}
  For a given $a\in(0,a_{\max})$, integrate the differential
  equation~\eqref{eq:X} from $(0,a)$ until $\bar s=\bar L_{\max}$. The
  set
  $(\bar s,u(\bar s),u'(\bar s))_{\bar s\in(0,\bar L_{\max})}$
  is then a 1D curve in $\mathbb{R}^3$. The collection of those curves
  for $a$ varying in $(0,a_{\max})$ is a 2D surface in $\mathbb{R}^3$,
  which we denote by $\calS$ (see Fig.~\ref{fig:3Dshooting}).

  There exists a homeomorphism $l:\calA\to\calS$.
\end{proposition}

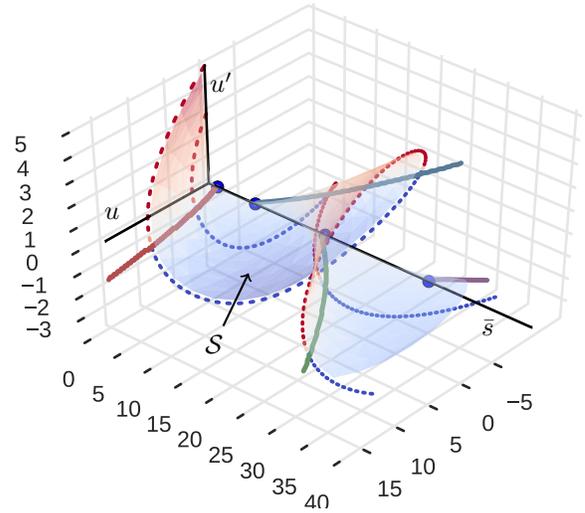
\begin{figure}[ht]
  \centering
  \begin{tikzpicture}
    \node[anchor=south west,inner sep=0] (image) at (0,0){
    \includegraphics[width=0.5\textwidth]{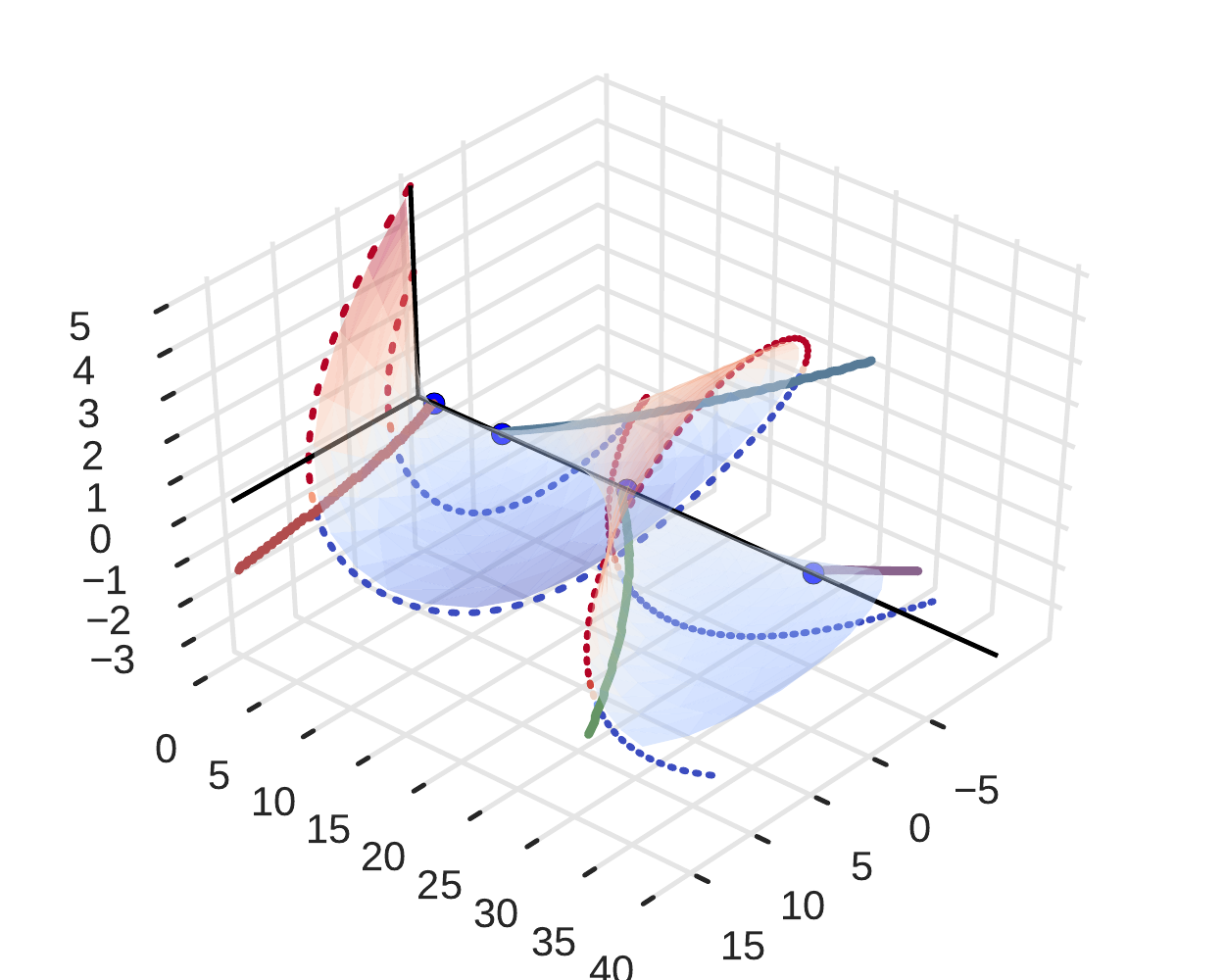}};
    \begin{scope}[x={(image.south east)},y={(image.north west)}]
      \node at (0.36, 0.78) {$u'$};
      \node at (0.2, 0.54) {$u$};
      \node at (0.75, 0.33) {$\bar s$};
      \node (a) at (0.35, 0.3) {$\cal S$};
      \draw [->, thick] (a) -- (0.4, 0.43);
    \end{scope}

  \end{tikzpicture}
  \caption{The surface $\calS$ that is homeomorphic to the
    configuration space $\calC$. We depict two solution curves on the
    surface $\calS$ (dashed lines), integrated from two different
    values of $a$ (large and medium). Red, blue, green, and purple
    lines represent respectively the first, second, third and fourth
    zero-radius loci (see Proposition~\ref{prop:zero}).}
  \label{fig:3Dshooting}
\end{figure}


\begin{proof}
  The construction of $l$ follows from the definition: given a pair
  $(a,\bar L)\in \calA$, integrate~(\ref{eq:X}) from $(0,a)$ until
  $\bar s=\bar L$. Then define $l(a,\bar L):=(\bar L,\vec u(\bar L))$.

  (1) Surjectivity of $l$. Consider a point $(\bar L,\vec u)
  \in\calS$. By definition of $\cal S$, there exists
  $a\in(0,a_{\max})$ so that integrating~(\ref{eq:X}) from $(0,a)$
  reaches $\vec u$ at $\bar s = \bar L$. Clearly,
  $l(a,\bar L)=(\bar L,\vec u)$.

  (2) Injectivity of $l$. This results from the Uniqueness theorem for
  ODEs, see Appendix~\ref{sec:ODE}.

  (3) Continuity of $l$. From the Continuity theorem for ODEs
  (Appendix~\ref{sec:ODE}), it is clear that the end point $(\bar
  L,\vec u(\bar L)) \in \calS$ depends continuously on the initial
  condition $a$.

  (4) Continuity of $l^{-1}$. Consider two points $(\bar L_1,\vec u^*_1),
  (\bar L_2,\vec u^*_2)\in \calS$ that are sufficiently close to each
  other, \ie
  \[
  |\bar L_1-\bar L_2| \leq \delta,\quad \|\vec u^*_1-\vec u^*_2\| \leq \delta,
  \]
  for some $\delta$ that we shall choose later. Consider the curves
  $\vec u_1,\vec u_2$ such that $\vec u_1(\bar L_1)=\vec u_1^*$ and
  $\vec u_2(\bar L_2)=\vec u_2^*$. By the Continuity theorem
  (Appendix~\ref{sec:ODE}) one has for some appropriate constant $K$, 
  \[
  \|\vec u_1(0)-\vec u_2(0)\| \leq e^{M\bar L_1}\|\vec u_1(\bar
  L_1)-\vec u_2(\bar L_1) \|
  \]
  \[
  \leq e^{M\bar L_1} \left( \|\vec u_1(\bar L_1)-\vec u_2(\bar L_2))\| +
    \|\vec u_2(\bar L_2)-\vec u_2(\bar L_1)\|\right)
  \]
  \[
  \leq  e^{M\bar L_1} (\delta + M|\bar L_1-\bar L_2|) = e^{M\bar L_1}(M+1)\delta,
  \]
  where the last inequality come from the uniform boundedness of
  $\vec u$. For any $\epsilon$, it suffices therefore to choose
  $\delta:=\frac{\epsilon e^{-M\bar L_1}}{M+1}$ so that
  $|a_1-a_2|=\|\vec u_1(0)-\vec u_2(0)\|\leq \epsilon$, which proves the
  continuity of $l^{-1}$.
\end{proof}

Combining Propositions~\ref{prop:A_to_C} and~\ref{prop:A_to_S}, we
obtain the following theorem.

\begin{theorem}
  \label{theo:homeomorphism}
  The configuration space $\calC$ of the rotating chain is
  homeomorphic to the 2D surface $\calS$ represented in
  Fig.~\ref{fig:3Dshooting}.
\end{theorem}


\subsection{Zero-radius loci and low-amplitude regime}

Before studying the stable subspaces, we need first to define the
zero-radius loci and the low-amplitude regime in the configurations
space. 

\begin{proposition}[and definition]
  \label{prop:zero}
  Zero-radius loci are configurations whose corresponding attachment
  radii verify $r=0$. Define $\bar L_i := L\omega_i^2/g$ where
  $\omega_i$ is the $i$-th discrete angular speed
  (Appendix~\ref{sec:lowamp}). We have the following properties on the
  surface $\calS$
  \begin{enumerate}[(i)]
  \item The $i$-th zero-radius locus is an infinite curve that
    branches out from the $\bar s$-axis at $(\bar L_i,0,0)$, see
    Fig.~\ref{fig:3Dshooting};
  \item The $i$-th zero-radius locus separates configurations in
    rotation mode $i-1$ from those in rotation mode $i$.
  \end{enumerate}
\end{proposition}

\begin{proof}
  (i) This property is implied by Kolodner's results, see the first
  paragraph of Sec~\ref{sec:number-configuration} for more details.

  (ii) Consider a rotation in mode $i-1$ and the corresponding curve
  $(\bar s, u_1(\bar s), u_1'(\bar s))_{\bar s\in [0, \bar L_1]}$.  By
  definition, $u_1'(\bar s)$ has $i-1$ zeros in the interval
  $[0, \bar L_{1}]$. Equivalently, we see that the 3D curve
  $(\bar s, u_1(\bar s), u_1'(\bar s))$ crosses the first,
  second\dots $i-1$-th zero-radius locus.  Now, since the loci start
  infinitely near the $\bar s$-axis [point (i)] and extend to
  infinity, any curve deformed from
  $(\bar s, u_1(\bar s), u_1'(\bar s))_{\bar s\in [0, \bar L_1]}$ also
  crosses the same loci.

  Consider now another rotation, which is in mode $i$, and the
  corresponding curve
  $(\bar s, u_2(\bar s), u_2'(\bar s))_{\bar s \in [0, \bar L_2]}$.
  By Theorem~\ref{theo:homeomorphism}, one can associate the two
  rotations with their endpoints
  $(\bar L_1, u_1(\bar L_1), u_1'(\bar L_1))$ and
  $(\bar L_2, u_2(\bar L_2), u_2'(\bar L_2))$ on the surface $\calS$.
  We will show that any continuous path that connect these two points
  necessarily crosses the $i$-th zero-radius locus. Indeed, assume the
  contradiction, it follows that there is a continuous curve ending at
  $(\bar L_2, u_2(\bar L_2), u_2'(\bar L_2))$ that does not cross the
  $i$-th locus. This is a contradiction to our assertion in the first
  paragraph of point (ii).
  
  We have thus established that the $i$-th zero-radius locus separates
  configurations of rotation mode $i-1$ from those in rotation mode
  $i$.
 \end{proof}

\begin{proposition}[and definition]
  \label{prop:1}
  The low-amplitude regime corresponds to configurations associated
  with infinitely small values of $u(\bar s)$ and $u'(\bar s)$, for
  all $\bar s\in(0,\bar L)$.
  \begin{enumerate}[(i)]
  \item The low-amplitude regime corresponds to points on the surface
    $\calS$ that are infinitely close to the $\bar s$-axis (in
    Fig.~\ref{fig:3Dshooting}).
  \item Moreover, this regime corresponds to points on
    the parameter space $\calA$ that have small values of $a$.
  \end{enumerate}
\end{proposition}

\begin{proof}

   (i) It is clear that a low-amplitude rotation has $u(\bar L)$ and
  $u'(\bar L)$ infinitely small. Conversely, if $u(\bar L)$ and
  $u'(\bar L)$ are infinitely small, by the continuity of the mapping
  $l^{-1}$ in the proof of Proposition~\ref{prop:A_to_S}, the initial
  condition $a$ is also infinitely small. Finally, integrating from an
  infinitely small $a$ will yield $u(\bar s)$ and $u'(\bar s)$
  infinitely small for all $\bar s\in(0,\bar L)$.

   (ii) This is true from (i).
\end{proof}

The low-amplitude rotations with zero attachment radius thus
correspond to $(\bar L_i,\delta u,0)$, $i\in\mathbb N$ for small
values of $|\delta u|$. In the sequel, we shall refer to the $i$-th
small-amplitude rotation with zero radius as the point
$(\bar L_i, 0, 0)$ instead of the more correct phase
``$(\bar L_i,\delta u,0)$ for small values of $|\delta u|$''.

\subsection{Stability analysis}
\label{sec:stability_analysis}

So far we have considered the space of all configurations of the
rotating chain, that is, all solutions to the equation of
motion~(\ref{eq:eom}). However, not all configurations are
\emph{stable}; in fact, experiments show that many are not. This
section investigates the structure of the stable subspace -- the
subset of stable configurations -- and discuss stable manipulation
strategies.

To analyze the stability of configurations, we model the chain by a
series of lumped masses, connected by stiff links, see
Fig.~\ref{fig:discrete-chain}.

\begin{figure}[ht]
  \centering
  \begin{tikzpicture}
    \node[anchor=south west,inner sep=0] (image) at (0,0)
    {\includegraphics[]{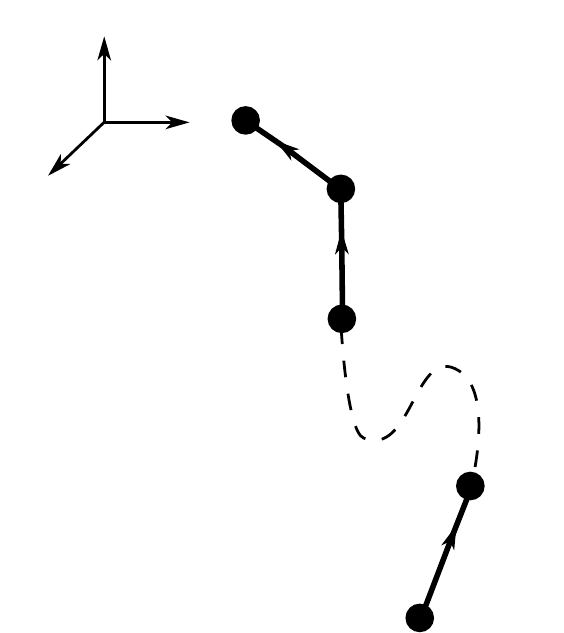}};
    \begin{scope}[x={(image.south east)},y={(image.north west)}]

      \node at (0.06, 0.68) {$x$};
      \node at (0.15, 0.95) {$z$};
      \node at (0.32, 0.75) {$y$};
      \node at (0.1, 0.825) {$\{\rm O\}$};

      \node at (0.85, 0.02) {$\vec x_0$};
      \node at (0.93, 0.22) {$\vec x_1$};
      \node at (0.74, 0.5) {$\vec x_{N-2}$};
      \node at (0.75, 0.7) {$\vec x_{N-1}$};
      \node at (0.52, 0.84) {$\vec x_{N}$};
      \node at (0.75, 0.145) {$\vec l_1$};
      \node at (0.5, 0.6) {$\vec l_{N-1}$};
      \node at (0.5, 0.725) {$\vec l_{N}$};
    \end{scope}

  \end{tikzpicture}
  \caption{\label{fig:discrete-chain} Discretized chain model with $N$ masses.
    }
\end{figure}
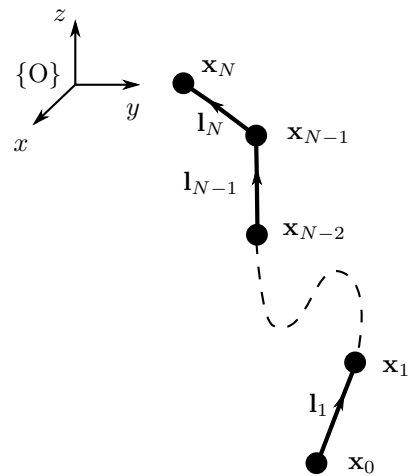

Denote the position of the $i$-th mass in the rotating frame
$\{\rm O\}$ by $\vec x_i\in \mathbb{R}^{3}$. The attached end is fixed
in $\{\rm{O}\}$ at $\vec x_N$. The state of the discretized chain is
then given by a $6N$-dimensional vector consisting of the positions
and velocities of the masses
\begin{equation}
  \label{eq:14}
  \vec y := [\vec x_0, \dot {\vec x}_0, \dots, \vec x_{N-1}, \dot {\vec x}_{N-1}].
\end{equation}
Applying Newton's laws to the masses (see details in
Appendix~\ref{sec:discrete}), one can obtain the dynamics equation
\begin{equation}
  \label{eq:discrete-dyn}
  \dot{\vec y} = \vec f(\vec y).
\end{equation}

From Proposition~\ref{prop:A_to_C}, the configurations of the rotating
chain can be represented by a pair $(a,\bar L)$, which is associated
with the position of the free end $\vec x_0$. Next, we discretize
$(0,a_{\max})\times(0,\bar L_{\max})$ into a 2D grid. For each
$(a,\bar L)$ in the grid, we integrate, from the free end $\vec x_0$,
the shape function of the discretized chain~(\ref{eq:14}) at
rotational equilibrium -- in the same spirit as in
Proposition~\ref{prop:A_to_C}. This discretized shape function
corresponds to a state vector
$\vec y^{\rm{eq}}:= [\vec x^{\rm{eq}}_0, \bfzero, \dots, \vec
x^{\rm{eq}}_{N-1}, \bfzero]$.
Finally, we assess the stability of $\vec y^{\rm{eq}}$ by looking at
the Jacobian
\[
\vec J(\vec y^{\rm{eq}}):=\frac{\rm d\vec f}{\rm d \vec y}(\vec y^{\rm{eq}}).
\]
Specifically, if the largest real part
$\lambda_{\max}:=\max_{i} \rm{Re}(\lambda_{i})$ of the eigenvalues of
$\vec J(\vec y^{\rm{eq}})$ is positive, then the system is unstable at
$\vec y^{\rm{eq}}$; if it is negative, then the system is
asymptotically stable at $\vec y^{\rm{eq}}$~\cite[Theorem
3.1]{khalil1996noninear}.

Fig.~\ref{fig:stability_map}(A) depicts the values of $\lambda_{\max}$
for $(a,\bar L)\in(0,5)\times(0,40)$. One can observe an interesting
distribution of these values; in particular, the sharp transitions
around the zero-radius loci (black lines). However, even though
$\lambda_{\max}$ gets very close to zero on the left side of the
zero-radius loci or in the low-amplitude regime, it is never negative,
hinting that the system is at best marginally stable. While this could
be expected from our model, which does not include any energy
dissipation, it is contrary to the experimental observation of stable
rotation states.

We need therefore to take into account aerodynamic forces in the chain
dynamics, see details in Appendix~\ref{sec:discrete}. Note that
aerodynamic forces do not significantly affect the analysis of the
previous sections, as their effect on the shape of the chain is
negligible: for example, for a chain of length $\SI{0.76}{m}$ and
parameters $(a, \bar L) = (2.0, 10.0)$, the changes in the equilibrium
positions are less than $\SI{1}{mm}$, which is 0.14\% of the chain
length.

Fig.~\ref{fig:stability_map}(B) depicts the values of $\lambda_{\max}$
for the system with aerodynamic forces. One can note that the overall
distribution of $\lambda_{\max}$ is very similar to that of the system
\emph{without} aerodynamic forces [Fig.~\ref{fig:stability_map}(A)],
but with the key difference that the regions in
Fig.~\ref{fig:stability_map}(A) with low but positive values now
contain in Fig.~\ref{fig:stability_map}(B) \emph{negative} values of
$\lambda_{\max}$, which corresponds to asymptotically stable states.

\begin{figure}[t]
  \centering
  \begin{tikzpicture}
    \node[anchor=south west,inner sep=0] (image) at (0,0)
    {\includegraphics[]{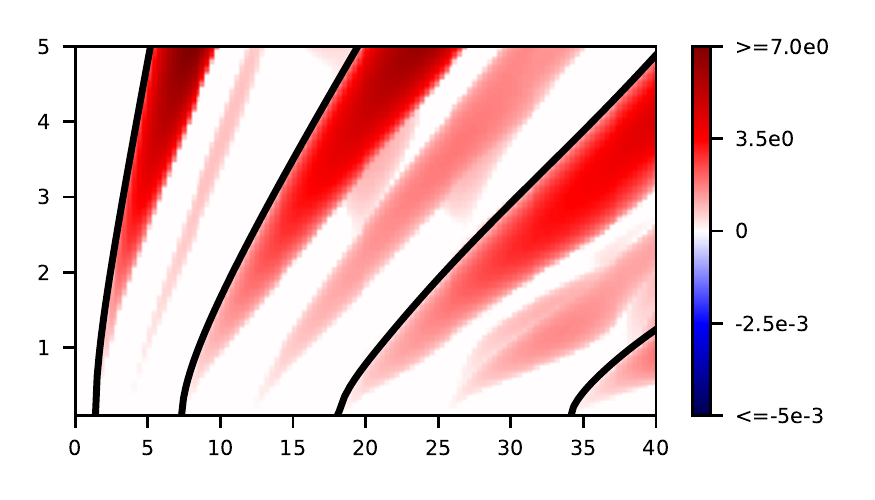}};
    \node[anchor=south west,inner sep=0] (imageright) at (0,-5)
    {\includegraphics[]{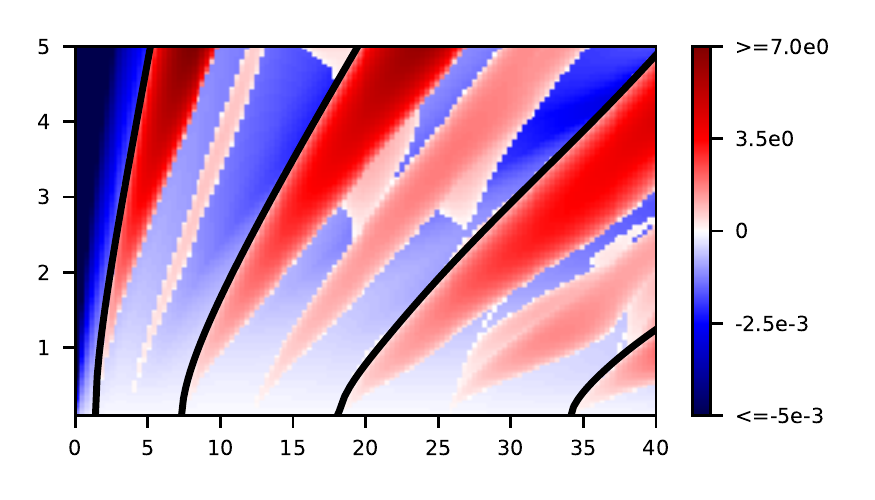}};
    \begin{scope}[shift={(imageright.south west)},
                  x={(imageright.south east)},y={(imageright.north west)}]
      \node at (0.12, 0.95) {(\textbf{B})};
      \node at (0.45, 0.02) {$\bar L$};
      \node at (0.0, 0.535) {$a$};

      \filldraw[fill=orange, thick] (0.3, 0.61)
      circle [x radius=0.018, y radius=0.03]
      node[anchor=south, very thick,
      color=green,
      xshift=-0.5cm] (start) {};
 
      \filldraw[fill=orange, thick] (0.42, 0.33)
      circle [x radius=0.018, y radius=0.03]
      node[anchor=south, very thick,
      color=orange,
      xshift=-0.5cm] (start) {};
      
      \draw[line width=2, ->, green, rounded corners] (0.29, 0.58)
      .. controls (0.23, 0.4) ..
      (0.19, 0.185) -- (0.365, 0.185) -- (0.4, 0.29);

      \draw[line width=2, dashed, ->, black] (0.32, 0.59)
      .. controls (0.45, 0.5) ..
      (0.425, 0.38)  ;
    \end{scope}
    \begin{scope}[x={(image.south east)},y={(image.north west)}]
      \node at (0.12, 0.95) {(\textbf{A})};
      \node at (0.45, 0.02) {$\bar L$};
      \node at (0.0, 0.535) {$a$};
    \end{scope}
  \end{tikzpicture}
  \caption{\label{fig:stability_map} (Best viewed in color) Maps of
    $\lambda_{\max}$, the largest real part of the eigenvalues of the
    linearized dynamics of two 10-link lumped-mass models at
    equilibrium: (\textbf{A}) model without aerodynamic forces and
    (\textbf{B}) model with aerodynamic forces. Positive values (red
    color) indicate unstable behaviors while negative values (blue
    color) indicate asymptotically stable behaviors. Most
    configurations that are stable in the presence of aerodynamic
    forces can not be concluded to be stable when there is no
    aerodynamic forces.  Black lines: zero-radius loci --
    configurations whose attachment radii are zeros. Green arrow: A
    path in the chain's configuration space that contains only stable
    configurations. Black dashed arrow: A path that contains unstable
    configurations. }
\end{figure}
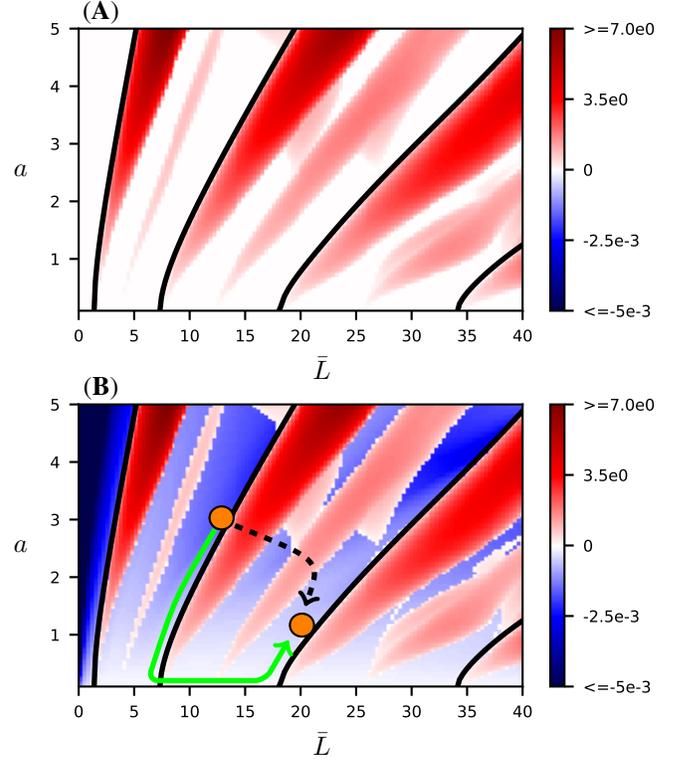

One can make three more specific observations:
\begin{enumerate}
\item Configurations that are immediately on the right-hand sides of
  the zero-radius loci and with $a$ relatively large are unstable (red
  color);
\item Configurations that are immediately on the left-hand sides of
  the zero-radius loci and with $a$ relatively large are
  stable (blue color);
\item Configurations with $a$ small (low-amplitude regime) are stable
  (light blue color).
\end{enumerate}

Observation (1) hints that the upper portions of the zero-radius loci
form ``unstable barriers'' in the configuration space. Therefore, it
is \emph{not} possible to stably transit between rotation modes $i-1$
and $i$ (which requires crossing the $i$-th zero-radius locus, see
Proposition~\ref{prop:zero}) while staying in the upper portion of the
configuration space [dashed black arrow in
Fig.~\ref{fig:stability_map}(B)]. Observation (2) implies that
transitions between configurations of the same mode can be
stable. Observation (3) hints that a possible transition strategy
might consist in (i) going down to the low-amplitude regime; (ii)
traversing the $i$-th zero-radius locus while remaining in the
low-amplitude regime; (iii) going up towards the desired end
configuration [green arrow in Fig.~\ref{fig:stability_map}(B)]. This
strategy thus traverses only regions with negative $\lambda_{\max}$
and can be expected to be stable. The next section experimentally
assesses this strategy.

\section{Manipulation of the rotating chain}
\label{sec:manipulation}

\subsection{Experiment}

We now experimentally test the manipulation strategy enunciated in the
previous section. More precisely, to stably transit between two
different rotation modes $i$ and $j$, we propose to [see the green
arrow in Fig.~\ref{fig:stability_map}(B)]

\begin{enumerate}
\item Move from the rotation of mode $i$ towards $(\bar L_{i+1},0,0)$
  while staying in the blue region of Fig.~\ref{fig:stability_map}(B);
\item Move along the $\bar L$-axis towards $(\bar L_{j+1},0,0)$;
\item Move from $(\bar L_{j+1},0,0)$ towards the rotation of mode $j$
  while staying in the blue region of blue region of
  Fig.~\ref{fig:stability_map}(B).
\end{enumerate}

In practice, the histories of the control inputs ($r$ and $\omega$) to
achieve the transitions in steps 1 and 3 can be found by simple linear
interpolation, see e.g.~Fig.~\ref{fig:exp_control}(A).

\begin{figure}[htp]
  \centering
  \textbf{A}\\
  \includegraphics[width=0.5\textwidth]{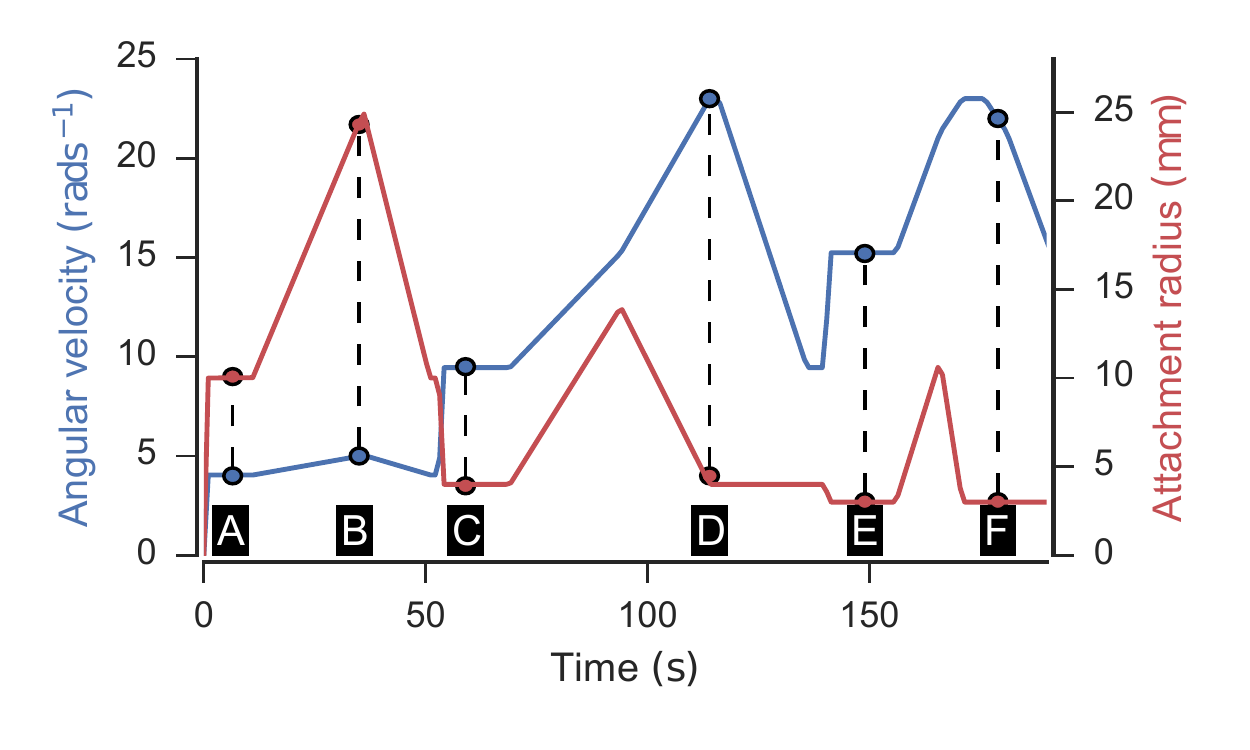}\\
  \textbf{B}
  \includegraphics[width=0.5\textwidth]{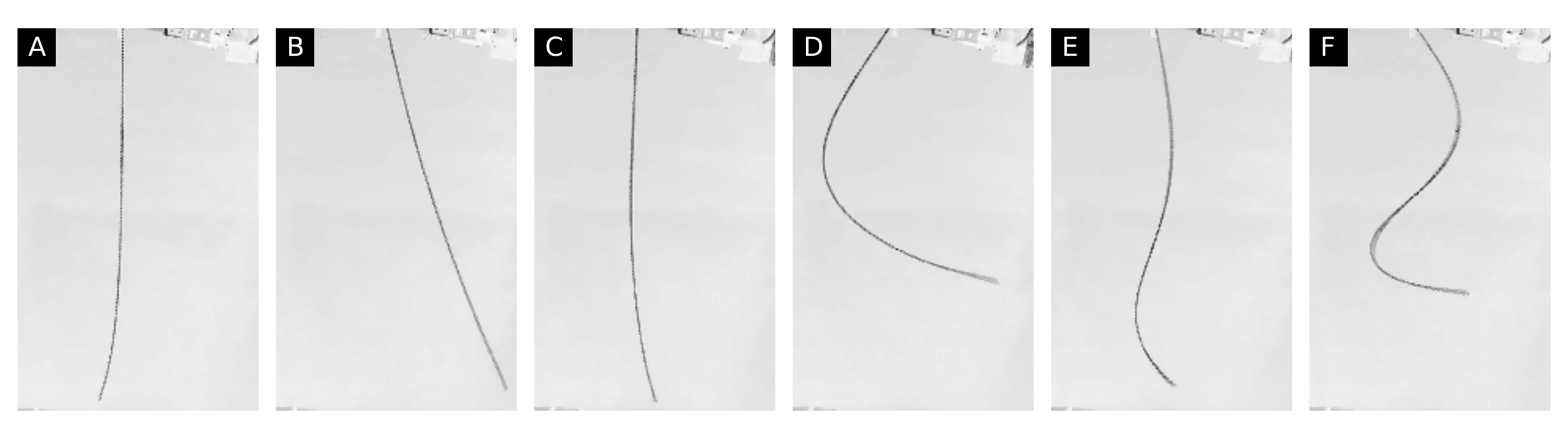}
  \caption{\textbf{A}: Histories of the control inputs. Red:
    attachment radius~$r$; blue: angular speed $\omega$. A:
    low-amplitude rotation at critical speed $\omega_1$. A $\to$ B:
    moving deep into rotation mode 0. B: stable rotation at mode 0. B
    $\to$ C: moving back to the low-amplitude regime with critical
    speed $\omega_1$ and subsequently increasing the speed to
    $\omega_2$ while staying in the low-amplitude regime. C:
    low-amplitude rotation at critical speed $\omega_2$.  C $\to$ D:
    moving deep into rotation mode~1.  D: stable rotation at mode
    1. E: low-amplitude rotation at critical speed $\omega_3$. F:
    stable rotation at mode~2. Note that the attachment radius was not
    exactly zero in the low-amplitude regimes, but set to some small
    values. This was necessary to physically generate the desired
    rotation speeds. \textbf{B}: Snapshots of the chain at
    different time instants. The labels A--F refer to the same time
    instants as in the control inputs plot. A video of the experiment
    (including more types of transitions) is available at
    \url{https://youtu.be/EnJdn3XdxEE}.}
  \label{fig:exp_control}
\end{figure}

We perform the following transitions 
\[
\mathrm{Rest} \to \mathrm {Mode\ 0} \to \mathrm {Mode\ 1} \to \mathrm      {Mode\ 2}
\]
on a metallic chain of length $\SI{0.76}{\meter}$ (note that the
weight of the chain is not involved in the calculations). The upper
end of the chain was attached to the end-effector of a 6-DOF
industrial manipulator (Denso VS-060). The critical speeds, calculated
using equation~(\ref{eq:critical_velocities}), are given in
Table~\ref{tab:critical}.

\begin{table}[htp]
\centering
\caption{Critical speeds for a chain of length $\SI{0.76}{\meter}$}  
\begin{tabular}{cccc}
    \toprule
     $i$                                 & $1$ & $2$ & $3$ \\
    \midrule
    $\omega_i$ (\SI{}{\radian\per\second}) & $4.34$ & $9.97$ & $15.64$ \\
    \bottomrule
\end{tabular}
\label{tab:critical}
\end{table}

A video of the experiment (including more types of transitions) is
available at
\url{https://youtu.be/EnJdn3XdxEE}. Fig.~\ref{fig:exp_control}(A)
shows the attachment radius and the angular speed as functions of
time. Fig.~\ref{fig:exp_control}(B) shows snapshots of the chain
at different rotation modes. As can be observed in the video, the
chain could transit between different rotation modes in a stable and
controlled manner. 

As the final note, we observed that any manipulation sequence that
traverses highly unstable regions (red regions in
Fig.~\ref{fig:stability_map}) definitely leads to unsustainable
rotations, as illustrated by the last section of the video.

\subsection{Implications for aerial manipulation}

For a circularly towing system, the ability to transit between
rotation modes is desirable. Indeed, different modes have different
functions. For instance, mode $0$ rotations are most suitable to
initiate a rotation sequence from a straight flying trajectory.  On
the other hand, rotations at higher order modes such as $1$ and $2$
have more compact shapes, smaller tip radii and higher tip velocities,
and are therefore more suitable to perform the actual deliveries or
explorations.

It is furthermore desirable to switch modes in a quasi-static manner,
as studied in this paper. Indeed, the \emph{transient} dynamics of a
heavily underactuated system such as the chain can be difficult to
handle. The infinite dimensionality of the system, unavoidable
modeling errors and aerodynamic effects make it challenging to design
and reliably execute non-quasi-static mode switching trajectories.

Our result suggests however that it is \emph{not} possible to realize
quasi-static mode transitions with fixed-wing aircraft. Indeed, since
the turning radii of such aircraft are lower-bounded, the resulting
rotations cannot enter the low-amplitude regime, which is necessary
for quasi-static mode transition, as shown in the above
development. Therefore, although non-quasi-static mode transitions are
more challenging to plan and execute, they must be studied in future
works.

\section{Conclusion}
\label{sec:conclusion}

The study of the rotating chain has a long and rich history. Starting
from the 1950's, a number of researchers have described its behavior,
and identified the existence of rotation modes. In this paper, we have
investigated for the first time the \emph{manipulation} problem, \ie
how to stably transit between different rotation modes. For that, we
developed a framework for understanding the kinematics of the rotating
chain with non-zero attachment radii and its configuration
space. Based on this understanding, we proposed a manipulation
strategy for transiting between different rotation modes in a stable
and controlled manner. In turn, on the practical side, this result has
some implications for aerial manipulation.

It can be shown (see Appendix~\ref{sec:tip-mass}) that all the
previous developments can be extended to the case of the chain with
non-negligible tip mass. The key enabling notion here is that
of~\emph{differential flatness}~\cite{murray1996trajectory}, with the
flat output being the state of the free end. By differential flatness,
given any trajectory of the free end, one can reversely compute the
state trajectory and the control trajectory of the whole system.  In
fact, the property that we have ``manually'' discovered in this paper
-- the configuration space of a rotating chain is parameterized by the
parameter space $\calA$ -- is related to the differential flatness of
the rotating chain system.  Indeed, each point $(a, \bar L)$
corresponds to a circular motion of the free end, which in turn, by
differential flatness, corresponds to the state and control trajectory
of the whole chain, which in turn defines the configuration. This
observation suggests two possible extensions:
\begin{itemize}
\item the motion of the free end can be more general (\eg an ellipse),
  and can thereby lead to more practical applications, such as
  swinging to hit some position with the tip mass;
\item other differentially-flat systems, whose flat output can be
  parameterized.
\end{itemize}

Another idea developed here, namely the visualization of the
configuration space based on forward integration of the shape
function, might find fruitful applications in the study of other
flexible objects with ``mode transition'', such as elastic rods or
concentric tubes subject to ``snapping''. Our future work will explore
these possible extensions.

\appendix

\subsection{Chain with non-negligible tip mass}
\label{sec:tip-mass}

Suppose that the free end of the chain carries a drogue of mass
$M$. We show that all the previous development can be applied to this
more general problem.

We first proceed similarly to Section~\ref{sec:background} and derive
the dynamics equation of the rotating chain with tip mass.  Writing the
force equilibrium equation at the tip mass yields
\begin{align}
  \label{eq:17}
  F(0) z'(0) &= Mg,\\
  \label{eq:19}
  F(0) \rho'(0) &= - M \rho(0) \omega^{2}.
\end{align}
Next, integrate Eq.~\eqref{eq:17} to obtain
\begin{equation}
  \label{eq:5}
  F(s) z(s)' = g(\mu  s + M),
\end{equation}
where $\mu$ is again the linear density of the chain. This equation
leads to
\begin{equation}
  \label{eq:6}
  F(s) = g\frac{\mu s + M}{\sqrt{1 - \rho'^2}}.
\end{equation}
One arrives at the governing equation
\begin{equation}
  \label{eq:7}
  \frac{\d {}}{\d s} \left(\rho' \frac{\mu  s + M}{\sqrt{1-\rho'^2}}
  \right) 
  + \rho \frac{\mu\omega^2}{g}
  = 0, 
\end{equation}
with boundary condition $\rho(L)=r$ where $r$ is the attachment
radius. One can now convert Eq.~\eqref{eq:7} to a dimensionless
equation 
\begin{align}
  \label{eq:9}
  \frac{\d{}^{2}u}{\d{\bar s}^{2}} + \frac{u}{\sqrt{(\bar s +
  M\omega^2 / \mu g)^2 + u^2}} = 0
\end{align}
by the following changes of variable
\begin{equation}
  \label{eq:8}
  \begin{aligned}
  u &:= \rho'\frac{\mu s + M}{\sqrt{1-{\rho'}^2}}\frac{\omega^2 }{\mu g}, \\ \quad \bar
  s &:= \frac{s\omega^2}{g}.
  \end{aligned}
\end{equation}
The boundary conditions are
\begin{align}
  \label{eq:10}
  u'(0) &= a,\\
  u(0) &= a \frac{M\omega^2}{\mu g},\\
  u'(\bar L) &= \bar r,
\end{align}
where $a= - \rho(0) \omega ^2 / g$ and $\bar r = - r
\omega^2/g$.


Eq.~\eqref{eq:9} is a BVP that can be solved using the shooting method
as described in Section~\ref{sec:shooting}. Moreover, we see that
$(a, \bar L)$ also parameterizes the solution space, which is the
configuration space of the rotating chain with tip mass.

\subsection{Low-amplitude regime}
\label{sec:lowamp}

Here we recall the results obtained by Kolodner~\cite{Kolodner1955}
for the low-amplitude regime. Low-amplitude rotations are defined by a
zero attachment radius $r=0$ and infinitely small values for the shape
function $\rho$.  Linearizing equation~(\ref{eq:x1}) about $\rho=0$
yields

\begin{equation}
  \label{eq:4}
  \rho w^2/g + \rho' + s \rho'' = 0,
\end{equation}
with the boundary condition $\rho(L) = 0$.

By a change of variable $v:=2\sqrt{s\omega^2/g}$, one can
rewrite the above equation as
\[
\rho v + \rho_v + \rho_{vv} v =0,
\]
which has solutions of the form
\[
  \rho(v) = c J_0(v),\quad \mathrm{\ie}
\]
\[
  \rho(s) = cJ_0(2\omega\sqrt{s/g}),
\]
where $J_0$ is the zeroth-Bessel function. The boundary condition
$\rho(L) = 0$ then implies that the angular speed can only take
discrete values $(\omega_i)_{i\in\mathbb{N}}$ where
\begin{equation}
  \label{eq:critical_velocities}
  \omega_i = \frac{h_i}{2} \sqrt{g/L}
\end{equation}
where $h_i$ is the $i$-th zero of the Bessel function $J_0$.

\subsection{Useful results from the theory of Ordinary Differential
  Equations}
\label{sec:ODE}

\begin{lemma}[Lipschitz]
  The ordinary differential equation~(\ref{eq:X}) satisfies Lipschitz
  condition in some convex bounded domain $\calD$ that contains
  $\calS$.
\end{lemma}
\begin{proof}
  Note first that $|u''(u, \bar s)| < 1$ for all $u, \bar s \in \bbR$,
  which implies that $\calS$ is bounded.  Set now
  \[
  \calD:= (0, \bar L_{\max}) \times (u_{\inf}, u_{\sup}) \times (u'_{\inf}, u'_{\sup}),
  \]
  where $u_{\inf}$, $u_{\sup}$, $u'_{\inf}$, $u'_{\sup}$ are bounds on
  $\calS$.  Clearly, $\calD$ is bounded, convex and contains $\calS$.
  Next, all partial derivatives $\frac{\p \bfX_i}{\p x_j}$ are
  continuous in $\calD$ (with continuation at $\bar s=0$, see Remark
  in Section~\ref{sec:dimless}). This implies that $\bfX$ is Lipschitz
  in $\calD$~\cite{HL2005elec}.
\end{proof}

We now recall two standard theorems in the theory of Ordinary
Differential Equations, see \eg \cite{HL2005elec}.

\begin{theorem}[Uniqueness]
  If the vector field $\bfX(\vec u, t)$ satisfies Lipschitz condition in
  a domain $\calD$, then there is at most one solution $\vec u(t)$ of
  the differential equation
  \[
    \frac{\d{\vec u}}{\d{t}} = \bfX(\vec u, t)
  \]
  that satisfies a given initial condition $\vec u(a) = \bfc \in \calD$.
\end{theorem}

\begin{theorem}[Continuity]
  Let $\vec u_1(t)$ and $\vec u_2(t)$ be any two solutions of the
  differential equation $\bfX(\vec u, t)$ in $T_1 \leq t \leq T_2$,
  where $\bfX(\vec u, t)$ is continuous and Lipschitz in some domain
  $\calD$ that contains the region where $\vec u_1(t)$ and $\vec u_2(t)$
  are defined. Then, there exists a constant $M$ such that
  \[
  \|\vec u_1(t) - \vec u_2(t)\| \leq e^{M|t-a|} \|\vec u(a) - \bfy(a)\|
  \]
  for all $a, t \in [T_1, T_2]$.
  
\end{theorem}

\subsection{The discretized chain model}
\label{sec:discrete}
Here we describe the procedure to obtain Eq.~\eqref{eq:discrete-dyn},
which is the dynamics equation of the discretized chain model employed
in Section~\ref{sec:stability_analysis}, see also
Fig.~\ref{fig:discrete-chain}. 

The net force $\vec F_{i}$ acting on the $i$-th mass is the sum of the
following three components:
\begin{enumerate}
\item \emph{fictitious forces}, which include the Coriolis force and
  centrifugal force associated with the rotating frame;
\item \emph{constraint forces} generated by the $i$-th and $i+1$-th
  links;
\item \emph{aerodynamic forces}, which include drag and lift.
\end{enumerate}
Fictitious forces are computed using standard formulas, which can be
found in any textbook on classical mechanics. To compute the
constraint forces, we model the links as stiff linear springs whose
stiffness approximates that of the chain used in the experiment of
Section~\ref{sec:manipulation}, which was
$\simeq\SI{8e7}{N/m}$. Constraint forces are then computed using
Hooke's law.

Next, to compute aerodynamic forces, we follow the modelling choices
of~\cite{Williams2007}, i.e. the aerodynamic forces acting on the
$i$-th link is placed entirely on the $i$-th mass. Specifically,
define the link length vector as
$\vec l_{i}:= \vec x_{i} - \vec x_{i-1}$ and denote by $\vec v_i$ the
actual air speed of the $i$-th mass, the angle of attack of the $i$-th
link is given by
\begin{equation*}
  \cos \xi_i = - \frac{\vec l_i \cdot \vec v_{i} } {\|\vec l_i\| \|\vec v_i\| }.
\end{equation*}
The drag and lift acting on the $i$-th link are then given by
\begin{equation*}
  \begin{aligned}
    \vec F_{i}^{D} &= 0.5 \rho_{a} C_{D} \|\vec l_i\| d \|\vec v_{i}\|^2\vec e_{D},\\
    \vec F_{i}^{L} &= 0.5 \rho_{a} C_{L} \|\vec l_i\| d \|\vec
    v_{i}\|^2\vec e_{L},
  \end{aligned}
\end{equation*}
where the directions and coefficents of drag and lift are
\begin{equation*}
  \begin{aligned}
    \vec e_{D} &= - \frac{\vec v_{i}}{\|\vec v_{i}\|}, &\quad
    \vec e_{L} &= - \frac{(\vec v_i\times \vec l_{i})\times \vec v_{i}}
    {\|(\vec v_i\times \vec l_{i})\times \vec v_{i}\|},\\
  C_{D} &= C_f + C_{n}\sin^{3} (\xi_i), &\quad C_{L} &= C_{n}\sin^2\xi_{i} \cos \xi_{i}.
  \end{aligned}
\end{equation*}
In the above equations, $d$ denotes the diameter of the chain, $C_{f}$
and $C_{n}$ are respectively the skin-fraction and crossflow drag
coefficients, $\rho_{a}$ is the air density. These parameters have the
following numerical values
\begin{equation*}
  \begin{aligned}
  d &= \SI{1}{mm},\; &\rho_a&= \SI{1.225}{kg/m^{3}}, \\
  C_f &= 0.038, \; &C_n &= 1.17. 
  \end{aligned}
\end{equation*}

Summing the components we obtain the $i$-th net force $\vec F_i$, from
which the acceleration of the $i$-th mass can be found as
\begin{equation*}
  \ddot {\vec x}_i = \vec F_i / m_i,
\end{equation*}
where $m_i$ is the mass of the $i$-th mass.  Rearranging the terms,
one obtains the dynamics equation \eqref{eq:discrete-dyn}
\begin{equation*}
  \dot{\vec y} = \vec f (\vec y).
\end{equation*}

\bibliographystyle{IEEEtran}
\bibliography{library}
\end{document}